\documentclass{article}

\usepackage[nonatbib,preprint]{neurips_2019}

\usepackage[utf8]{inputenc} 
\usepackage[T1]{fontenc}    

\usepackage{graphicx}
\usepackage{mathtools, fancyhdr,amsmath,amssymb,amsthm,enumitem}
\setlist{leftmargin=4.2mm}

\newcommand{\argmax}{\operatornamewithlimits{arg\,max}}
\DeclarePairedDelimiter\abs{\lvert}{\rvert}
\DeclarePairedDelimiter\norm{\lVert}{\rVert}
\DeclarePairedDelimiter\dotp{\langle}{\rangle}
\newcommand{\1}{\mathbf{1}}
\newtheorem{theorem}{Theorem}

\usepackage{hyperref}       
\usepackage{url}            
\usepackage{booktabs}       
\usepackage{amsfonts}       
\usepackage{nicefrac}       
\usepackage{microtype}      

\usepackage[backend=bibtex,style=numeric-comp,sorting=none]{biblatex}
\newenvironment{myquote}%
{\list{}{\leftmargin=0.16in\rightmargin=0in}\item[]}%
{\endlist}

\usepackage{subcaption,caption,wrapfig}
\DeclareMathOperator{\diag}{diag}

\title{Batch Active Learning \\ Using Determinantal Point Processes}

\author{%
  Erdem B\i y\i k\\
  Electrical Engineering\\
  Stanford University\\
  Stanford, CA 94305\\
  \texttt{ebiyik@stanford.edu} \\
  \And
  Kenneth Wang \\
  Computer Science and Physics\\
  Stanford University\\
  Stanford, CA 94305\\
  \texttt{kwang411@stanford.edu} \\
  \AND
  Nima Anari\\
  Computer Science\\
  Stanford University\\
  Stanford, CA 94305\\
  \texttt{anari@cs.stanford.edu}\\
  \And
  Dorsa Sadigh\\
  Computer Science and Electrical Engineering\\
  Stanford University\\
  Stanford, CA 94305\\
  \texttt{dorsa@cs.stanford.edu}
}

\begin{document}

\maketitle

\begin{abstract}
  Data collection and labeling is one of the main challenges in employing machine learning algorithms in a variety of real-world applications with limited data. While active learning methods attempt to tackle this issue by labeling only the data samples that give high information, they generally suffer from large computational costs and are impractical in settings where data can be collected in parallel. Batch active learning methods attempt to overcome this computational burden by querying batches of samples at a time. To avoid redundancy between samples, previous works rely on some ad hoc combination of sample quality and diversity. In this paper, we present a new principled batch active learning method using Determinantal Point Processes, a repulsive point process that enables generating diverse batches of samples. We develop tractable algorithms to approximate the mode of a DPP distribution, and provide theoretical guarantees on the degree of approximation. We further demonstrate that an iterative greedy method for DPP maximization, which has lower computational costs but worse theoretical guarantees, still gives competitive results for batch active learning. Our experiments show the value of our methods on several datasets against state-of-the-art baselines.
\end{abstract}

\section{Introduction}
\vspace{-5px}
\label{sec:introduction}
The availability of large datasets has played a significant role in the success of machine learning algorithms in variety of fields ranging from societal networks to computer vision. 
However, in some fields, such as speech recognition \cite{varadarajan2009maximizing}, text classification \cite{cuong2013active}, image recognition \cite{sener2018active}, and robotics \cite{biyik2018batch,christiano2017deep,sadigh2017active,akrour2012april,jain2015learning}, collecting and labeling data can be time-consuming and costly, as they require interactions with humans or physical systems.

This means we need to either look for ways to collect large amounts of labeled data or develop methods that reduce the labeling effort. In this paper, we focus on the latter problem, where we investigate a general purpose active learning algorithm that could be used in a variety of applications. In the active learning framework, the user is kept in the learning loop. While there are several different variants of human-in-the-loop learning systems \cite{shivaswamy2015coactive,jain2015learning,jamieson2011active,fails2003interactive,lang2016feasibility}, we are interested in a model that asks for the labels of intelligently selected data samples. In this way, the model finds a good optima using much fewer samples \cite{chen2013near,sadigh2017active}. We refer to \cite{settles2009active} for a survey on earlier active learning works.

In the classical setting, active learning frameworks select a single data sample at each iteration. However, a single data sample is likely to have very little impact on most of the modern learning models \cite{sener2018active}. And more importantly, each iteration requires retraining of the model, which makes parallel labeling inapplicable and the computational cost a new challenge \cite{biyik2018batch}. As labeling requires direct interaction with humans, large computation times are undesirable.

These problems were tackled using batch active learning, which enables the labeling of several data points at a time~\cite{azimi2012batch,biyik2018batch,sener2018active}. With full generality, we put the batch active learning problem as follows.

\begin{myquote}
	\textbf{Problem Definition}: 
	We have an unlabeled dataset $\mathcal{X}\!\subseteq\!\mathbb{R}^{N\!\times\!d}$ of $N$ samples from $C$ classes. Can we train a high-accuracy classifier by labeling only $K$ samples, with batches of size $k$ per iteration?
\end{myquote}
While we define the problem for classification, it can be easily extended to regression problems, too.

Batch-mode active learning is special in that we cannot select the individually most informative data samples as the batch, because the samples can share a lot of information with each other, and this leads to highly suboptimal batches in practice, despite known theoretical bounds \cite{chen2013near}. In fact, the actual solution to the optimization involved in batch generation is known to require an exhaustive search \cite{guo2008discriminative}, so researchers generally rely on different approximations and heuristics.


\begin{figure*}[t]
	\centering
	\includegraphics[width=0.6\textwidth]{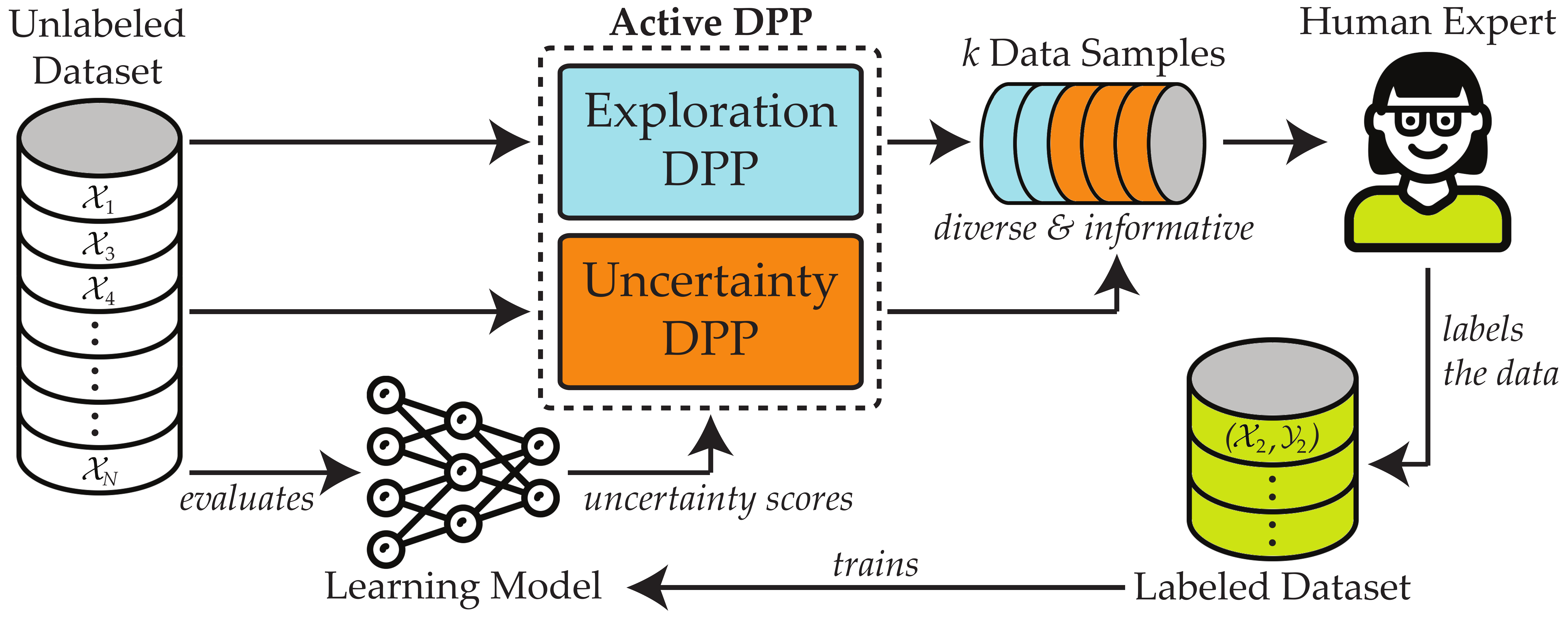}
	\caption{Our Active DPP method. Unlabeled samples and their uncertainty scores with respect to the learning model are given to our algorithm, which utilizes two DPPs. Uncertainty DPP selects data samples based on the uncertainty scores. Exploration DPP aims at finding new decision boundaries. They output a batch of samples that are both separately and jointly diverse and informative. These samples are then labeled by a human expert and the learning model is trained.}
	\vspace{-15px}
	\label{fig:frontfig}
\end{figure*}

In this paper, we present a new batch-active learning algorithm using Determinantal Point Processes (DPP), a class of repulsive point processes that is especially useful for generating diverse batches \cite{kulesza2012determinantal} and has been employed for several different machine learning applications over the past decade \cite{kulesza2011k,kulesza2010structured,gillenwater2012discovering,kathuria2016batched,dupuy2018learning,xie2017deep,mariet2018exponentiated,zhang2017determinantal,zhang2019active}. While the general idea of promoting diversity in batches for active learning is not novel \cite{biyik2018batch,ravi2018meta,cardoso2017ranked,yang2015multi,houlsby2011bayesian,patra2011batch,joshi2009multi,zhdanov2019diverse}, previous works relied on ad hoc combinations of two quantities that represent \emph{informativeness} and \emph{diversity} of data samples. 
Our key insight is that DPPs can be used to formalize a principled approach to balance informativeness and diversity; an approach that is not only easy to implement, but also competitive with the state-of-the-art methods. We visualize the overall framework in Fig.~\ref{fig:frontfig}.



The main contributions of this paper are:
\begin{itemize}[nosep]
	\item We present a theoretical method to approximate the mode of any DPP distribution based on a new rounding algorithm integrated with previously known convex relaxation methods \cite{nikolov2015randomized}.
	\item  We prove our new \emph{maximum coordinate rounding algorithm} matches the best possible approximation ratio without requiring the computationally expensive method of conditional expectations.
	\item We develop a novel batch active learning algorithm that selects batches as samples from or an approximate mode of a DPP distribution augmented with appropriate score values.
	\item We demonstrate our results in classification tasks on several datasets, along with results in preference-based reward learning presented in the Appendix due to space constraints.
\end{itemize}


\vspace{-5px}
\section{Background}
\vspace{-5px}
\label{sec:background}
A point process is a probability measure on a ground set $\mathcal{X}$ over finite subsets of $\mathcal{X}$. In accordance with our problem definition, we will have $\abs{\mathcal{X}}=N$.

An $L$-ensemble defines a DPP through a real, symmetric and positive semidefinite (PSD) $N$-by-$N$ kernel matrix $L$ \cite{borodin2005eynard}. Then, sampling a subset $X=A\subseteq \mathcal{X}$ has the probability
\begin{align}
P(X=A) \propto \det{L_A}
\label{eq:dpp}
\end{align}
where $L_A$ is an $\abs{A}$-by-$\abs{A}$ matrix that consists of the rows and columns of $L$ that correspond to the items in $A$. For instance, if $A=\{i,j\}$, then
\begin{align*}
P(X=A)\propto L_{ii}L_{jj} - L_{ij}L_{ji}.
\end{align*}
We can consider $L_{ij}=L_{ji}$ as a similarity measure between the dataset items $i$ and $j$. The nonnegativeness of the second term in the above expression shows an example of repulsiveness property of DPPs. This property makes DPPs the ubiquitous tractable point process to model negative correlations.

As $\det{L_A}$ can be positive for various $A$ with different cardinalities, we do not know $\abs{A}$ in advance. There is an extension of DPPs referred to as $k$-DPP where it is guaranteed that $\abs{A}=k$, and Eq.~\eqref{eq:dpp} remains valid \cite{kulesza2011k}. In this work, we employ $k$-DPPs and refer to them as DPPs for the rest of the paper for simplicity. The complexity of exact sampling from $k$-DPPs is equal to $O(N^{\omega} + Nk^3)$ where $O(N^\omega)$ is the complexity associated with matrix multiplication and is due to the necessary eigendecomposition on $L$. The output of the sampling is a subset that consists of $k$ values that are more diverse (less similar) than the uniform sampling (except the trivial values of $k$, or diagonal $L$).

The running time of $O(N^{\omega} + Nk^3)$ is very slow in practice, because we generally have large datasets, which correspond to large $N$. To overcome this issue, several approaches have been proposed, such as using a dual representation \cite{kulesza2010structured} that relies on the assumption that $\textrm{rank}(L)\ll N$, or adopting random projections \cite{gillenwater2012discovering} which relies on the result of \cite{johnson1984extensions} that the distances between high dimensional points can be approximately preserved after a logarithmic number of random projections.

In this work, we use a Markov Chain Monte Carlo (MCMC) method to approximate the sampling \cite{anari2016monte,li2016fast,mariet2018exponentiated, ALOV18a}, which will be easier and relatively more practical. In this method, a starting set of points is first selected to maximize the likelihood (see Section~\ref{sec:dpp_mode}). Then, several Monte Carlo steps are taken by removing and inserting one sample to the set for the mixing of the Markov chain.

Now, we explain what parameters we can have in an $L$-ensemble DPP. Since $L$ is known to be PSD, there always exists a matrix $D$ such that $L=D^\top D$. Then, we note that
\begin{align*}
P(X=A)\propto \det L_A = \textrm{Vol}^2(\{D_i\}_{i\in A}),
\end{align*}
so the probability is proportional to the square of the associated volume. In fact, by using a generalized version of DPP, we can approximately achieve:
\begin{align}
P(X=A)\propto \textrm{Vol}^{2\alpha}(\{D_i\}_{i\in A}).
\label{eq:alpha}
\end{align}
Previous work has shown the fast mixing guarantees of MCMC methods with $0\leq\alpha\leq1$ \cite{ALOV18a}, and the utility of exponentiated DPPs \cite{mariet2018exponentiated}. One can note that higher $\alpha$ enforces more diversity, because the probability of more diverse sets (larger volumes) will be boosted against the less diverse sets.

To construct $L$ for batch active learning, we further define the columns of $D$ as
\begin{align*}
D_i = q_i\phi_i
\end{align*}
where $q_i\in\mathbb{R}_{\geq0}$ is the \emph{score} of $i^\textrm{th}$ item that represents how much we want that item in our batch. We will use it to weight the samples based on how informative they are for the learning model. $\phi_i$ is a vector of similarity features and normalized such that $\norm{\phi_i}_2=1$. Defining a matrix $S$ such that $S_{ij}=\phi_i^\top\phi_j$, we introduce another parameter $\gamma$ that is related to the score values $q_i$:
\begin{align}
L_{ij} = q_i^{\gamma/\alpha}S_{ij}q_j^{\gamma/\alpha}.
\label{eq:gamma}
\end{align}
In this way, by increasing $\gamma$ for fixed $\alpha$, we give more importance to the scores while sampling.

One last important point for our work is that conditioning a DPP distribution still results in a DPP. That is, $P(X=A\cup B | B\subseteq X)$ is distributed according to a DPP with a transformed kernel:
\begin{align*}
L' = \left(\left[(L + I_{\bar B})^{-1}\right]_{\bar B}\right)^{-1} - I
\end{align*}
where $\bar B = \mathcal{X}\setminus B$, $I$ is the identity matrix, and $I_{\bar B}$ is the projection matrix with all zeros except at the diagonal entries $(i,i)$ for $\forall i\in {\bar B}$ where the entry is $1$.

\vspace{-5px}
\section{Approximating the Mode of a DPP}
\vspace{-5px}
\label{sec:dpp_mode}
With proper tuning of $\alpha$ and $\gamma$, the batches that are both diverse and informative will have higher probabilities of being sampled. This motivates us to find the mode of the distribution\footnote{$A^*$ is called the mode of the DPP distribution if $A^*=\argmax_A P(X=A)$.}, which will guarantee informativeness and diversity. Another advantage of using the mode, instead of a random sample from the distribution, is the fact that it is significantly faster, because MCMC sampling already attempts to compute the mode to choose the starting point.

However, finding the mode of a DPP is NP-hard \cite{ko1995exact}. It is hard to even approximate better than a factor of $2^{ck}$ for some $c\!>\!0$, under a cardinality constraint of size $k$ \cite{civril2013exponential}. Here, we first discuss two different algorithms to approximate the mode: Greedy and Convex Relaxation. Greedy algorithm suffers from poor approximation ratio, and convex relaxation algorithm is computationally prohibitive as it has super-linear dependence on $N$. We then present our novel \emph{maximum coordinate rounding algorithm} that matches the best possible approximation ratio without requiring the computationally expensive method of conditional expectations.

\textbf{Greedy Algorithm.}
One approach to approximate DPP-mode is greedily adding samples to the batch. More formally, to approximate
\begin{align*}
\argmax_A P(X=A) = \argmax_A \textrm{Vol}^{2\alpha}(\{D_i\}_{i\in A}),
\end{align*}
we greedily add samples to $A$. Let $A^{(m)}$ denote the set of selected samples at iteration $m$. We have
\begin{align*}
A^{(m+1)} = A^{(m)} \cup \{\argmax_j\textrm{Vol}^{2\alpha}(\{D_i\}_{i\in A^{(m)}\cup\{j\}})\},
\end{align*}
which we repeat until we obtain $k$ elements in $A$. \cite{ccivril2009selecting} showed that the greedy algorithm always finds a $k^{O(k)}$-approximation to the mode.

\textbf{Convex Relaxation Algorithm.}
The greedy algorithm does not provide the state-of-the-art approximation guarantee. \cite{nikolov2015randomized} showed that one can find an $e^k$-approximation to the mode by using a convex relaxation. We present the algorithm of \cite{nikolov2015randomized} stated in an equivalent form: Formally, consider the generating polynomial associated to the DPP:
\[ g(v_1,\dots,v_N)=\sum_{A:\abs{A}=k}\det(L_A)\prod_{i\in A}v_i.  \]
Finding the mode is equivalent to maximizing $g(v_1,\dots,v_N)$ over \emph{nonnegative integers} $v_1,\dots,v_N$ satisfying the constraint $v_1+\dots+v_N=k$. We get a relaxation by replacing integers with nonnegative reals, and using the insight that $\log(g)$ is a concave function which can be maximized efficiently:
\[ \max\left\{\log g(v_1,\dots,v_N)\;\vert\; v_1+\dots+v_N=k \right\}. \]
If $v_1^*,\dots, v_N^*$ is the maximizer, one can then choose a set $A$ of size $k$ with $P(A)\propto \prod_{i\in A} v_i^*$. Then $E[\det(L_A)]$ will be an $e^k$-approximation to the mode. Although this approximation holds in expectation, the probability that the sampled $A$ is an $e^k$-approximation can be exponentially small. To resolve this, \cite{nikolov2015randomized} resorted to the method of conditional expectations, each time deciding whether to include an element in the set $A$ or not.

The main drawback of this method is its computational cost. In particular, the running time of the methods that compute $g$ scale as a super-linear polynomial in $N$, which is problematic for the typical use cases where $N$ is large. Computing $g$ and $\nabla g$ is needed for solving the relaxation as well as running the method of conditional expectations.

\vspace{-5px}
\subsection{Maximum Coordinate Rounding}
\vspace{-5px}
We instead propose a new algorithm that avoids the method of conditional expectations. We also propose a heuristic method to find the maximizers $v_1^*,\dots,v_N^*$ by stochastic mirror descent, where each stochastic gradient computation requires sampling from a DPP. Approximate sampling from DPPs can be done in time $O(N\cdot k^2 \log k)$, scaling linearly with $N$ \cite{hermon2019modified}. Our algorithm is:
\begin{enumerate}[nosep]
	\item Find the nonnegative real maximizers $v_1^*,\!\dots,v_N^*$ of $\log g(v_1,\!\dots,v_N)$ subject to $v_1+\dots+v_N\!=\!k$.
	\item Let $v_i^*$ be the maximum among $v_1^*,\dots,v_N^*$. Put $i$ in $A$, and recursively find $k-1$ extra elements to put in $A$, working with the conditioned DPP.
\end{enumerate}
\begin{theorem}
	The above algorithm finds an $e^k$-approximation of the mode.
\end{theorem}
The proof is by induction on $k$. We prove that there is only a factor of at most $e$ lost at each iteration of the second step. We provide the full proof, as well as the details of the stochastic mirror descent algorithm in the Appendix. We also provide an empirical comparison that shows the superior performance of the maximum coordinate rounding algorithm over greedy algorithm in the Appendix.

\vspace{-5px}
\section{Methods for Batch Active Learning}
\vspace{-5px}
\label{sec:methods}
Armed with the methods of constructing a DPP kernel $L$ that ensures diversity and informativeness, and approximating the mode of any DPP, we are now ready to present our DPP-based batch active learning methods. We start with describing some simple baselines in order to build ideas to finally introduce our \emph{Active DPP} methods. We first start with passive methods that select all the samples in the beginning. We later introduce active methods which iteratively select a small number of samples based on the information from previously selected and labeled data samples.

\textbf{Uniform Sampling.}
The most straightforward way to approach the problem is to take a uniformly random subset of data samples, have them labeled, and train a model using this subset. The problems associated with this na\"ive approach are: 1) Some of the samples, possibly the majority, will be almost completely redundant due to the shared information, 2) some parts of the space of data samples can be given more importance than the rest due to randomness. The second issue occasionally hurts and occasionally improves the training, but the former one almost surely hurts.

\textbf{Passive DPP.}
In this approach, our idea is to take the random subset using a DPP with $q_i=1\;\forall i\in\{1,\dots,N\}$, so that the samples will homogeneously cover the space of data samples (see Fig.~\ref{fig:labeled_samples}). This approach solves the unequal importance problem of uniform sampling that we described, although the performance may not be improved. It also mitigates the negative effects of the redundancy problem, because the samples will be more distant from each other, so they will be more informative on average. Hence, we can expect improved performance over uniform sampling.

\begin{figure*}[t]
	\centering
	\includegraphics[width=\textwidth]{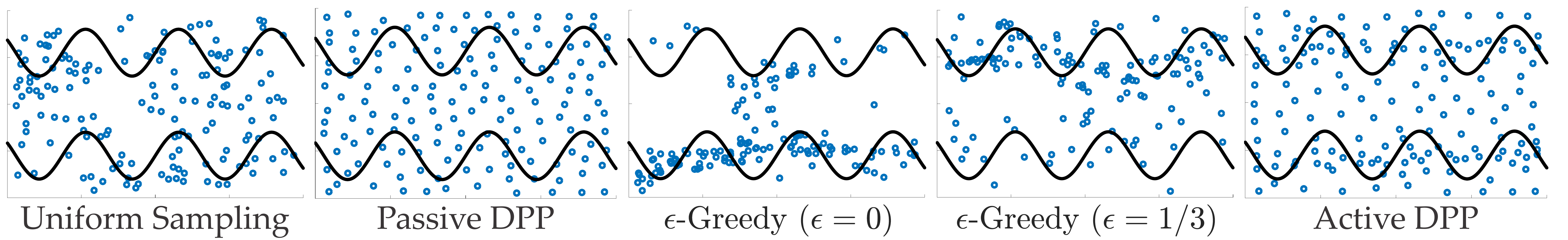}
	\vspace{-10px}
	\caption{Labeled samples within each method are shown. Black lines are the true decision boundaries. Uniform sampling method samples very close points and leads to redundancy. Passive DPP increases the distance between the samples. $\epsilon$-Greedy with no exploration takes more samples near the decision boundaries it could detect, but fails to detect some boundaries. $\epsilon$-Greedy with some exploration mitigates this problem but still has high redundancy. Active DPP overcomes this by enforcing diversity, while still capturing more samples near the decision boundaries.}
	\label{fig:labeled_samples}
	\vspace{-20px}
\end{figure*}

\textbf{Passive DPP-Mode.}
The approximate-mode of the DPP distribution is used as the batch instead of random sampling within Passive DPP approach.

\textbf{$\boldsymbol{\epsilon}$-Greedy.}
The methods we have described so far were all passive in the sense that they never utilized user feedback. The idea in active learning is to select the new samples based on the previously selected and labeled ones. For example, in the widely adopted uncertainty sampling, the model is first trained using only the samples that have already been labeled. Then, the new samples for labeling are selected based on how uncertain the model is on each of the unlabeled samples.

We will give explicit definitions of what uncertainty measures we use in Section~\ref{sec:technical_details}.

Two major drawbacks of uncertainty sampling algorithm are as follows. First, it requires re-training for each and every sample of $K$ samples. This poses a computational limitation. The second major issue is the lack of exploration. We expect high uncertainties near true decision boundaries. Therefore, the algorithm might always focus on the region near a spotted decision boundary and so might miss the other decision boundaries that can be far away. This is known as \emph{bias in active learning} and is well-observed \cite{bach2007active, dasgupta2008hierarchical, beygelzimer2009importance}. This phenomenon has been recently theoretically analyzed by \cite{mussmann2018uncertainty}.

Batch-mode active learning methods can be employed to overcome both issues. By selecting a batch of samples at a time, the computation burden problem is significantly reduced \cite{biyik2018batch}. We can also incorporate an $\epsilon$-greedy strategy to introduce exploration to handle the latter problem:
\begin{itemize}[nosep]
	\item We add the most uncertain $(1-\epsilon)k$ unlabeled data samples to the batch.
	\item We then uniformly randomly select $\epsilon k$ samples out of the remaining ones for exploration.
\end{itemize}
Figure~\ref{fig:labeled_samples} shows selected samples with $\epsilon$-Greedy when $\epsilon\in\{0,1/3\}$. The benefit of exploration can be clearly seen: The algorithm with $\epsilon=0$ misses an important portion of the decision boundary.

One crucial problem associated with $\epsilon$-greedy method is that although it actively selects the samples, it ignores the information shared by the correlated samples. As it selects $(1-\epsilon)k$ samples based only on uncertainty, it is very likely to select very correlated samples. This is because correlated samples tend to give close uncertainty values to each other as it can be seen from Fig.~\ref{fig:labeled_samples}. And having correlated samples in the batch leads to high redundancy.

\vspace{-5px}
\subsection{Active DPP-based Methods}
\vspace{-5px}
\textbf{Active DPP.}
We propose a DPP-based batch active learning algorithm to resolve all the problems associated with the aforementioned methods. For that, we incorporate \emph{dissimilarity} and \emph{uncertainty} values into DPP distribution for \emph{diversity} and \emph{informativeness}.

For \emph{dissimilarity}, we can set $S_{ij}=\phi_i^\top\phi_j$ either by constructing similarity features for each data sample, or by using the distance between the samples (assuming such a distance metric $h$ exists) with a Gaussian kernel, which is known to be PSD:
\begin{align}
S_{ij} = \exp\left(-\frac{h(\mathcal{X}_i,\mathcal{X}_j)^2}{2\sigma^2}\right)
\label{eq:similarity_matrix}
\end{align}
where $\sigma$ is a hyperparameter. For example, in \cite{wang2018active}, the authors used weighted Euclidean distance and adaptively set the weights by learning the importance of each dimension.

For \emph{uncertainty}, we set sample scores $q_i$ to be the uncertainty values while constructing the DPP kernel matrix $L$. Then, the hyperparameter $\gamma$ in Eq.~\eqref{eq:gamma} represents how much we care about uncertainties.

As sampling from this DPP creates diverse sets, it simultaneously enforces \emph{diversity} and \emph{informativeness} with proper tuning of $\alpha$ and $\gamma$, for which we describe our procedure in the Appendix.

To solve the lack of exploration, we again utilize DPPs. We construct another DPP kernel with $\gamma\!=\!0$ to have all the remaining samples equally important\footnote{For $\alpha=0$, Eq.~\eqref{eq:gamma} becomes ill-defined. We take scores as $1$, and the sampling reduces to uniform sampling.}. We take $\epsilon k$ samples from this exploration DPP.

While DPPs reduce within-batch redundancy, it is also important to mitigate the redundancy among different batches. Moreover, we can improve exploration by querying unexplored regions of the space more. To achieve both, we use the fact that conditioning a DPP still results in a DPP distribution. Therefore, we condition all the DPPs to contain the samples that are already selected. And we sample $(1-\epsilon)k$ (or $\epsilon k$ for exploration) more samples from this conditioned DPP for the current iteration.

\textbf{Active DPP-Mode.}
Similar to Passive DPP-Mode, we introduce a mode variant of Active DPP. In this method, we approximate the mode, for both exploration and uncertainty DPPs.




\vspace{-5px}
\section{Experiments \& Results}
\vspace{-5px}
\label{sec:results}
In this section, we present our experiments with classification tasks on synthetic and several different real datasets. We also experimented with preference-based reward learning on $4$ different robotics tasks \cite{palan2019learning,sadigh2016planning,todorov2012mujoco,wise2016fetch} and observed Active DPP-Mode method significantly outperforms all $4$ competitor methods \cite{biyik2018batch} in all $4$ tasks, except for one insignificant comparison. Due to space constraints, we present preference-based reward learning experiments and its theoretical convergence guarantees in the Appendix. In all classification experiments, $K=150$ and the problem is cold-start (no known labels in the beginning). We set the batch size $k=15$.

We make the codes for classification, preference-based reward learning and the maximum coordinate rounding algorithm to approximate the mode of a DPP publicly available for reproducibility\footnote{See \url{https://github.com/Stanford-ILIAD/DPP-Batch-Active-Learning}.}.

\vspace{-5px}
\subsection{Implementation Details}
\vspace{-5px}
\label{sec:technical_details}
For dissimilarity, we use the Gaussian kernel approach as in Eq.~\eqref{eq:similarity_matrix}. Specifically, we use Euclidean distance for $h$ to construct the similarity matrix $S$.

For uncertainty scores $q_i$, one can use entropy over the model-generated probabilities of samples' belonging to each class in classification tasks. In this work, to have better estimates of uncertainty, we use an ensemble of $10$ neural networks and calculate the entropy based on the mean probabilities over the ensemble \cite{tibshirani1996comparison}, which is common in practice \cite{christiano2017deep,lakshminarayanan2017simple}. Other alternatives to estimate uncertainty in classification include Bayesian neural networks \cite{depeweg2018decomposition}, the margin measure that quantifies the difference between top two class predictions \cite{ebert2012ralf,joshi2009multi}, or the distance to the decision boundary \cite{balcan2007margin}. For regression tasks, one can again use entropy or the variance of the model estimates. For preference-based reward learning, the expected volume removal \cite{sadigh2017active,biyik2019green} can be used.

For mode-approximation, even though maximum coordinate rounding could lead to higher overall performances, we resort to the greedy algorithm for our main experiments, as it is considerably faster.

We describe the tuning procedure we followed for the hyperparameters in the Appendix.

\vspace{-5px}
\subsection{Classification on a Synthetic Dataset}
\vspace{-5px}
\label{sec:synthetic_dataset}
\begin{figure}[t]
	\centering
	\includegraphics[width=0.67\columnwidth]{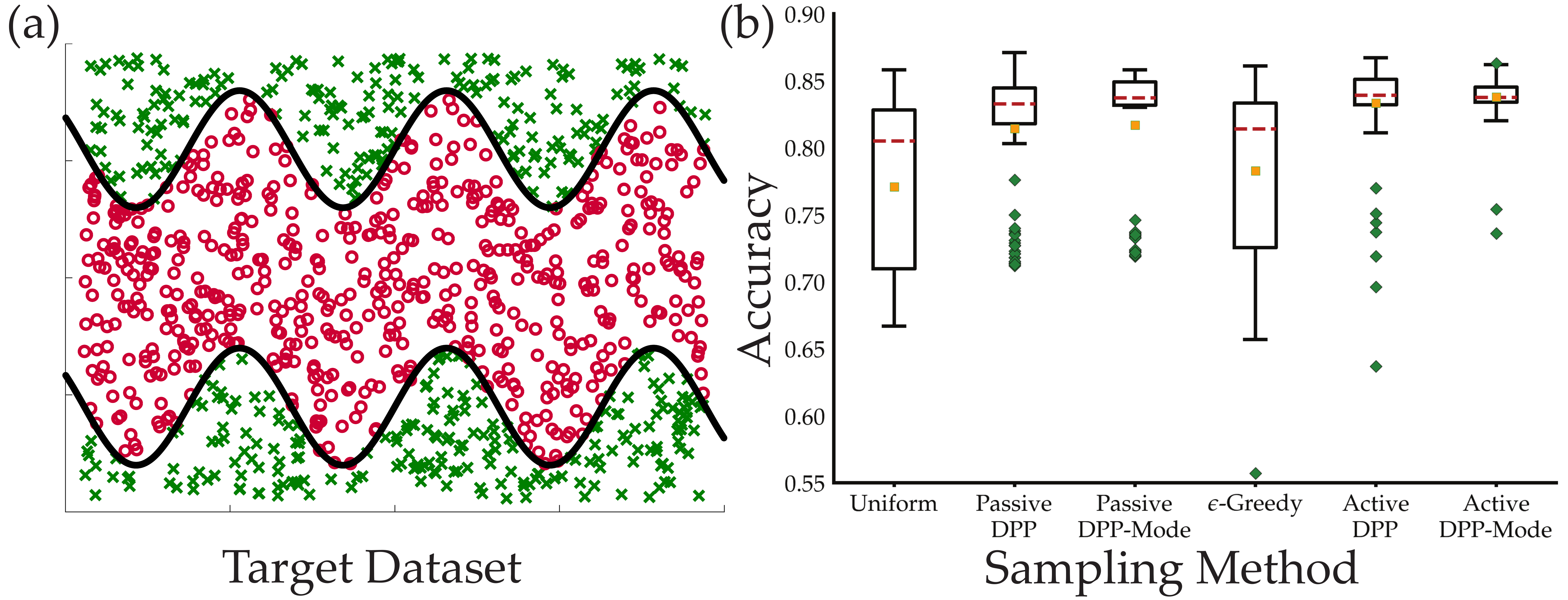}
	\caption{\textbf{(a)} The synthetic dataset. Red (circle) and green (cross) points are the members of two different classes. \textbf{(b)} Accuracies obtained by $100$ runs on synthetic dataset are shown.}
	\label{fig:synthetic_datasets}
	\vspace{-15px}
\end{figure}

We created a synthetic dataset with $C=2$, $d=2$ and $N=1000$ to quantitatively evaluate our methods. The dataset consists of points each of which is labeled depending on whether it is between two sine functions or not. Fig.~\ref{fig:synthetic_datasets}(a) visualizes this dataset. We also created a $1000$-sample dataset $(\mathcal{X}_{\textrm{test}},\mathcal{Y}_{\textrm{test}})$ from the same distribution for accuracy assessments. Note that using the remaining $850$ samples of $\mathcal{X}$ for assessment would be unfair, as active methods are likely to leave easier samples out.

We used a small classifier model to avoid overfitting: a feed-forward one-hidden layer neural network. The hidden layer has $4$ nodes. We used the sigmoid function as the nonlinear activation.

We then performed $100$ runs of each method on $\mathcal{X}$. The results are shown in Fig.~\ref{fig:synthetic_datasets}(b). As it can be seen, Passive DPP significantly outperforms uniform sampling ($p\!<\!0.005$, two-sample t-test), because it reduces the correlation between the selected samples. Passive DPP-Mode does not improve much in terms of performance even though it also significantly outperforms uniform sampling ($p\!<\!0.005$). However, it gives comparable performance in a more time-efficient way. $\epsilon$-Greedy performs poorly by selecting redundant samples (see Fig.~\ref{fig:labeled_samples}) and is significantly outperformed by Passive DPP and its greedy variant ($p\!<\!0.005$). Active DPP outperforms all aforementioned methods significantly ($p\!<\!0.005$) by enforcing both diversity and informativeness. Furthermore, Active DPP-Mode gives very similar accuracy results to Active DPP with lower computation time, as explained before.

We also provide the representative visuals of the selected samples for uniform sampling, Passive DPP, $\epsilon$-Greedy, and Active DPP in Fig.~\ref{fig:labeled_samples}. It can be seen that Passive DPP homogeneously cover the space of samples. While $\epsilon$-Greedy and Active DPP capture more samples near decision boundaries; unlike $\epsilon$-Greedy, Active DPP  selects samples that are not very close to each other.

\vspace{-5px}
\subsection{Classification on Real Datasets}
\vspace{-5px}
\label{sec:real_dataset}
We conducted experiments on MNIST dataset \cite{lecun1998gradient} for hand-written digit classification and Fashion-MNIST \cite{xiao2017fashion} for clothing classification (both $C=10$, $d=784$). Similar to \cite{zhang2019active}, we took a subset of the training data for faster computation. Specifically, our $\mathcal{X}$ and $\mathcal{X}_\textrm{test}$ consist of $5000$ and $10000$ random samples, respectively \textemdash so they can be imbalanced. To avoid overfitting and to show the importance of diversity, we used an autoencoder to reduce the number of dimensions to $d'=5$, separately for both datasets. Our autoencoder had two ReLU hidden layers with $128$ and $64$ nodes, respectively. We call these datasets ``Compressed MNIST" and ``Compressed Fashion-MNIST".

We also downloaded $4$ other datasets from OpenML\footnote{Retrieved April 2019, \url{https://www.openml.org/}} \cite{vanschoren2013openml}: Wall Robot Navigation (WRN, \cite{freire2009short,dua2019uci}), Image Segmentation (Segment, \cite{dua2019uci}), Morphological \cite{jain2000statistical,dua2019uci}, and Blood Transfusion Service Center (Blood, \cite{yeh2009knowledge}). We divided these datasets into two halves as training and test sets. We normalize the datasets prior to training.

We compared our method Active DPP-Mode, which performed significantly better than the other baselines, with the following established methods \footnote{See \url{https://github.com/google/active-learning}.}:
\begin{itemize}[nosep]
	\item Greedy variant of Sener'18 \cite{sener2018active}, which tries to minimize the upper bound on the loss function of learning by posing the problem as a core-set selection problem,
	\item Graph density method of Ebert'12 \cite{ebert2012ralf}, which attempts to choose the most representative samples,
	\item Hsu'15 \cite{hsu2015active} which attempts to select the right batch active learning method as a bandit problem among Sener'18, Ebert'12, and uniform sampling,
	\item Dasgupta'08 \cite{dasgupta2008hierarchical}, which performs hierarchical clustering over the unlabeled data and exploits the information in the pruned clustering tree to select the batches, and
	\item Xu'03 \cite{xu2003representative}, which clusters the unlabeled data and selects the medoids for labeling.
\end{itemize}

We performed $100$ runs of each method on $\mathcal{X}$ where the classifier models are feed-forward neural networks with sigmoid nonlinearities, whose structures are described in the Appendix. Fig.~\ref{fig:real_results} shows the results. As in many other works \cite{ebert2012ralf,hsu2015active,zhdanov2019diverse}, we empirically observed no algorithm is superior in all datasets. However, our Active Greedy-DPP is one of the best methods with statistical significance on $4$ of the $6$ datasets. It also gave pretty competitive accuracy values on the other two.

Our active DPP-based methods can be further improved by tuning $\epsilon$ through some heuristics or expertise about the dataset of interest. The low performance of the graph density method of Ebert'12 might be because it focuses only on finding the set of most representative samples of the dataset. The performance of Hsu'15 implies combining Ebert'12 with the methods using uncertainty can improve its performance. Lastly, although Sener'18 proposed a general method, it was mainly developed for convolutional neural networks, and it can perform better when high-quality features are available.

\begin{figure}[t]
	\centering
	\includegraphics[width=\columnwidth]{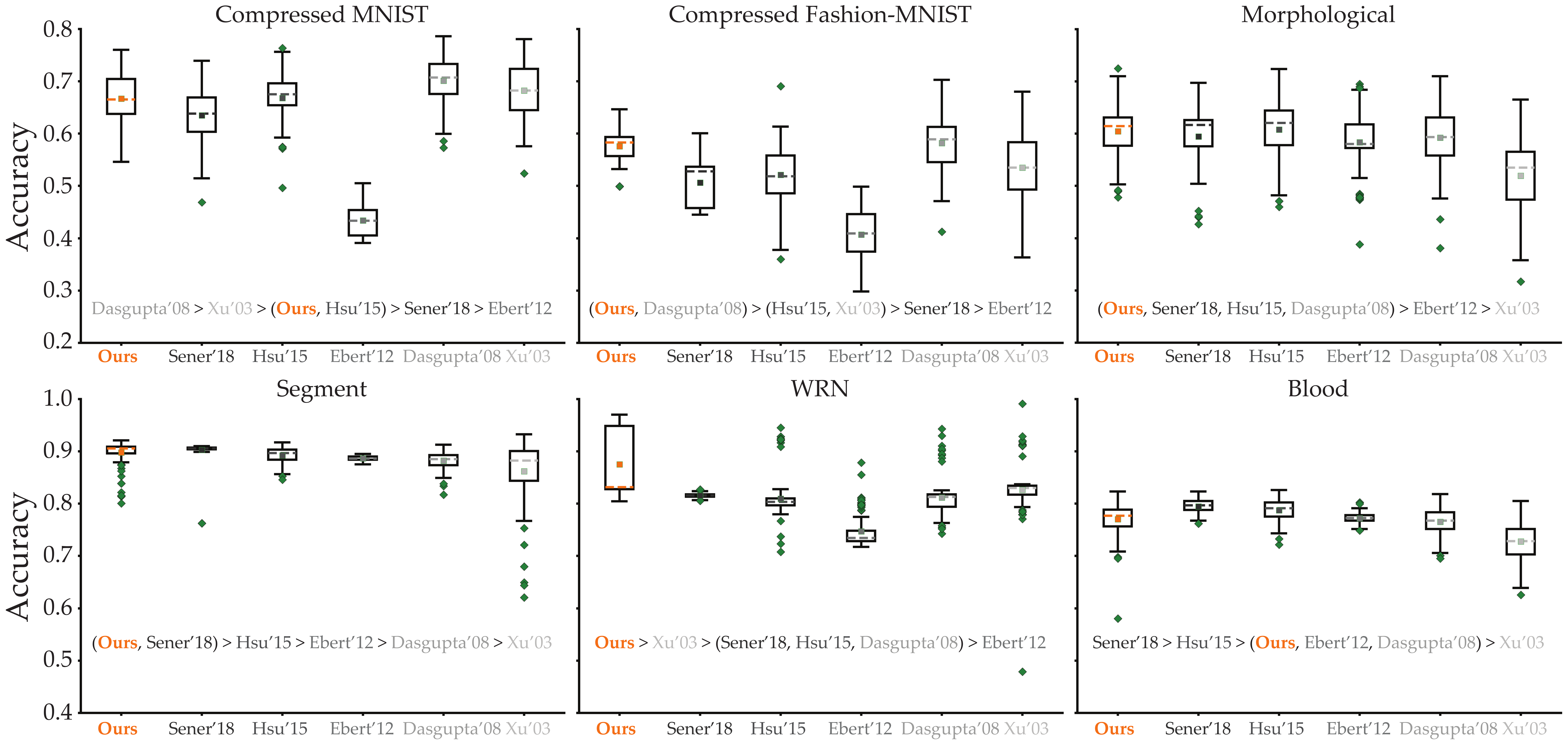}
	\vspace{-10px}
	\caption{Accuracies obtained by $100$ runs on real datasets with our Active DPP-Mode and other baselines are shown. Significant performance differences ($p<0.05$) are noted. Tuples with parentheses indicate the differences between the methods are insignificant.}
	\label{fig:real_results}
	\vspace{-15px}
\end{figure}

\vspace{-5px}
\section{Discussion}
\vspace{-5px}
\label{sec:discussion}
\textbf{Summary.} In this work, we proposed a batch-mode active learning method and demonstrated our results on classification tasks with synthetic and several real-world datasets. We emphasize our proposed framework is more general. For example, we can use our framework to fit a Gaussian Process through actively selected batches, or for reward learning as we demonstrate in the Appendix.

\textbf{Limitations.} One limitation of our approach is computational cost. Although MCMC sampling and performing only rank-one updates significantly accelerate the process, sampling still poses a bottleneck when $N$ is large. To overcome this issue, one might consider using the dual representation of DPPs with careful design of similarity features \cite{kulesza2012determinantal}, and/or the idea of random projections.

\textbf{Future Work.} Some future directions we currently consider involve the use of DPPs for regression and for generative models, which can be quite useful in the cases where data generation is costly, such as medical imaging. Other DPP-related research direction we are considering is diversity based learning where DPPs can be employed as part of the loss functions. This can be useful in natural language processing, such as for text summarization, or in video understanding.

\textbf{Conclusion.} In this paper, we proposed a DPP-based batch active learning method which can be used in a wide variety of domains where data labeling is costly. While our method is very general and applicable to different types of problems, we demonstrated a few use cases with classification tasks.

We believe due to their natural property of providing diversity, DPPs can be used for other problems in machine learning. We hope our work stimulates interest and leads to various applications of DPPs.


\printbibliography

\section{Appendix}
\label{sec:appendix}
\subsection{Details of the Maximum Coordinate Rounding Algorithm}
\label{sec:details_convex_relax}

\textbf{Proof of the Approximation Ratio}

Here we provide the proof of Theorem 1, showing that our rounding algorithm achieves the approximation ratio of $e^k$.

\begin{proof}[Proof of Theorem 1]
	We prove this by induction on $k$. We simply prove that each time we select an element and put it in $A$, we only lose a factor of at most $e$. Note that the first-order optimality condition of $v_1^*,\dots,v_N^*$ means that
	\[ \nabla \log g(v_1^*,\dots,v_N^*)=c\1-\sum_{j:v_j^*=0}c_je_j,  \]
	for some $c$ and collection of $c_j\geq 0$. Here $\1$ is the all-ones vector and $e_1,\dots,e_N$ are the standard basis vectors. By complementary slackness, we have $c_jv_j^*=0$ for all $j$. Since $v_i^*>0$, it must be that $c_i=0$, and it follows that $c=\norm{\nabla \log g(v_1^*,\dots,v_N^*)}_\infty=\partial_i \log g(v_1^*,\dots,v_N^*)$.
	Note that $g$ is a $k$-homogeneous polynomial and it follows that $\dotp{\nabla g(v), v}=kg(v)$. Applying the inequality $\dotp{\nabla g, v}\leq \norm{\nabla g}_\infty\cdot \norm{v}_1$, we get
	\[ kg(v_1^*,\dots,v_N^*)\leq \norm{\nabla g(v_1^*,\dots,v_N^*)}_\infty\cdot \norm{v^*}_1,  \]
	Noting that $\norm{v^*}_1=k$ and $\norm{\nabla g(v_1^*,\dots,v_N^*)}_\infty=\partial_i g(v_1^*,\dots,v_N^*)$, we get
	\[ \partial_i g(v_1^*,\dots,v_N^*)\geq g(v_1^*,\dots,v_N^*). \]
	But note that $\partial_i g$ is exactly the generating polynomial for the conditioned DPP (where we condition on $i\in A$). So it is enough to show that $\max \partial_i g(u_1,\dots,u_N)$ over $u_1+\dots+u_N=k-1$ is at least $1/e$ times the above amount. To do this we simply let $u_j^*=(k-1)v_j^*/(k-v_i^*)$ for $j\neq i$ and we set $u_i^*=0$. Since $\partial_i g$ is $(k-1)$-homogeneous we get
	\begin{align*}
	\partial_i g(u_1^*,\dots,u_N^*)=\left(\frac{k-1}{k-v_j}\right)^{k-1}\partial_i g(v_1^*,\dots,v_N^*)\geq
	\left(\frac{k-1}{k}\right)^{k-1}g(v_1^*,\dots,v_N^*).
	\end{align*}
	We conclude by noting that $((k-1)/k)^{k-1}\geq 1/e$.
\end{proof}

\textbf{Stochastic Mirror Descent Algorithm}

In this section we propose a stochastic mirror descent algorithm to optimize the following convex program over nonnegative reals
\[ \max\left\{\log g(v_1,\dots,v_N)\;\vert\; v_1+\dots+v_N=k \right\}, \]
where $g$ is the generating polynomial associated to a $k$-DPP, i.e.,
\[ g(v_1,\dots,v_n)=\sum_{A:\abs{A}=k}\det(L_A)\prod_{i\in A} v_i. \]

Our proposed algorithm is repetitions of the following iteration:
\begin{enumerate}[nosep]
	\item Sample a set $A$ with $P(A)\propto \prod_{i\in A} v_i \det(L_A)$.
	\item Let $u\leftarrow v+\eta \1_A$, where $\1_A$ is the indicator of $A$.
	\item Let $v\leftarrow ku/(\sum_i u_i)$.
\end{enumerate}

Note that the sampling in step 1 can be done by MCMC methods, since we are sampling $A$ according to a DPP. Careful implementations of the latest MCMC methods (e.g. \cite{hermon2019modified}) run in time $O(N\cdot k^2\log k)$ time. The parameter $\eta$ is the step size and can be adjusted.

Now we provide the intuition behind this iterative procedure. First, let us compute $\nabla \log g$. We have
\[ \frac{\partial_i g(v_1,\dots,v_N)}{g(v_1,\dots,v_N)}=\frac{1}{v_i}\frac{\sum_{A:i\in A}\det(L_A)\prod_{j\in A} v_j}{\sum_{A}\det(L_A)\prod_{j\in A}v_j}, \]
but this is equal to $P(i\in A)/v_i$. Therefore $\nabla \log g=\diag(v)^{-1}p$, where $p$ is the vector of marginal probabilities, i.e. $p_i=P(i\in A)$. Note however that $E[\1_A]=p$. So this suggests that we can use $\diag(v)^{-1}\1_A$ as a stochastic gradient.

Numerically we found $\diag(v)^{-1}\1_A$ to be unstable. This is not surprising as $v$ can have small entries, resulting in a blow up of this vector. Instead we use a stochastic mirror descent algorithm, where we choose a convex function $\phi$ and modify our stochastic gradient vector by multiplying $(\nabla^2 \phi)^{-1}$ on the left.

We found the choice of $\phi(v_1,\dots,v_N)=\sum_i v_i \log v_i$ to be reasonable. Accordingly, we have $\nabla^2 \phi =\diag(v)^{-1}$, and therefore
\[ (\nabla^2 \phi)^{-1}\diag(v)^{-1}\1_A=\1_A. \]
Finally, note that step 3 of our algorithm is simply a projection back to the feasible set of our constraints (according to the Bregman divergence imposed by $\phi$).

\textbf{Choice of Stochastic Gradient Vector}

Note that the vector $\1_A$ in step 2 of the algorithm can be replaced by any other random vector $X$, as long as the expectation is preserved. One can extract such vectors $X$ from implementations of MCMC methods \cite{hermon2019modified,ALOV18a}. The MCMC methods that aim to sample a set $A$ with probability proportional to $\prod_{i\in A} v_i \det(L_A)$ work as follows: starting with a set $A$, one drops an element $i\in A$ chosen uniformly at random, and adds an element $j$ back with probability proportional to $\det(L_{A-i+j})$, in order to complete one step of the Markov chain. We can implement the same Markov chain, and let $X_j$ be $k$ times the probability of transitioning from $A-i$ to $A-i+j$ in this chain. It is easy to see that if the chain has mixed and $A$ is sampled from the stationary distribution
\[ \mathbb{E}[X]=\mathbb{E}[\1_A]. \]
We found this choice of $X$ to have less variance than $\1_A$ in practice.

\textbf{Empirical Comparison with Greedy Algorithm}

Here we provide an empirical comparison between the performance of the greedy algorithm versus our maximum coordinate rounding algorithm.

We used two sets of experiments. In the first, we generated $200$ random points inside $[0,1]^2$, and used a Gaussian kernel with parameter $\sigma=1$ and attempted to find the mode of the $k$-DPP for $k=3$. In the second, we generated $200$ random points inside $[0,1]^2$ and used a Gaussian kernel with parameter $\sigma=0.2$ and attempted to find the mode of the $k$-DPP for $k=20$. We ran each experiment $100$ times (each time generating a new set of random points).

The results can be seen in Fig.~\ref{fig:convex-comparison}. We plotted $\det(L_A)$ vs.\ $\det(L_B)$, where $A$ is the set returned by the greedy algorithm, and $B$ is the set returned by the maximum coordinate rounding algorithm.

\begin{figure}
	\centering
	\begin{subfigure}{.5\textwidth}
		\centering
		\includegraphics[width=\linewidth]{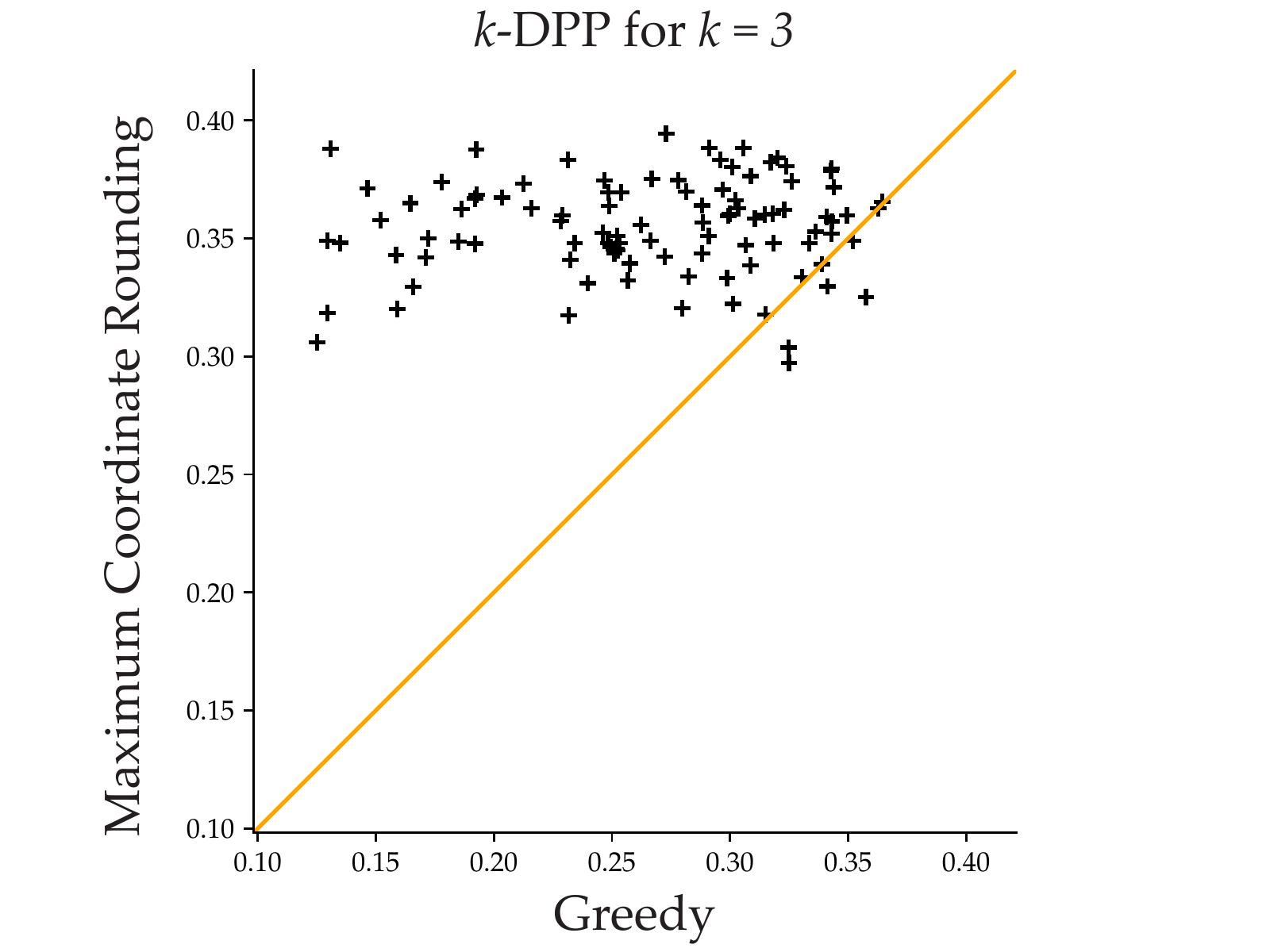}
	\end{subfigure}%
	\begin{subfigure}{.5\textwidth}
		\centering
		\includegraphics[width=\linewidth]{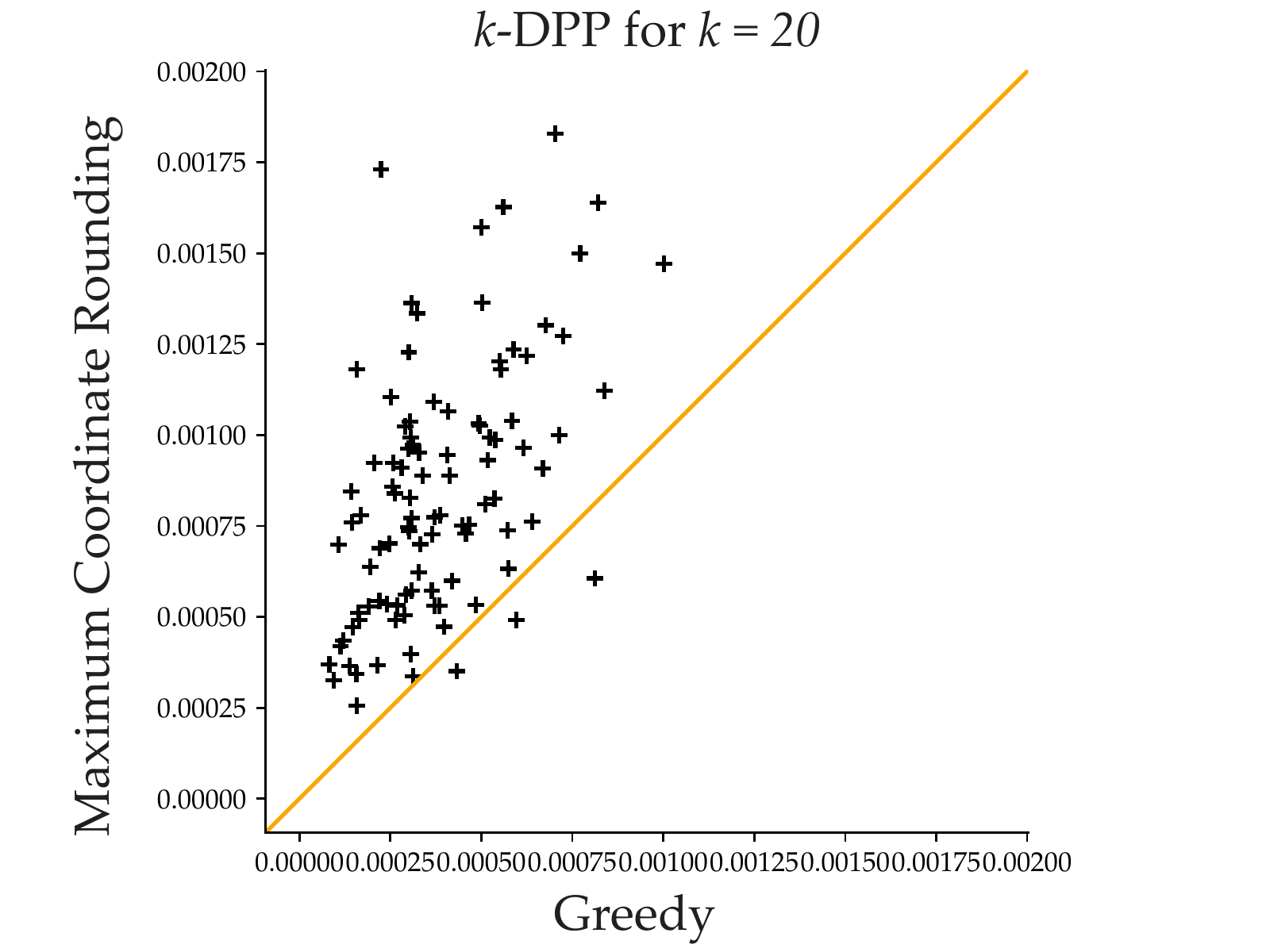}
	\end{subfigure}
	\caption{Comparison of the greedy algorithm and the maximum coordinate rounding algorithm. In 93\% of the $k\!=\!3$ cases, and in 97\% of the $k\!=\!20$ cases, our method returns a better or equal solution.}
	\label{fig:convex-comparison}
	\vspace{-10px}
\end{figure}

\subsection{Preference-based Reward Learning}
Reward learning is a vital problem in reinforcement learning and has many applications specifically in robotics. Recently, \cite{sadigh2017active} showed it is possible for a robot to learn the operator human's reward function by querying him/her with pairwise trajectory comparisons in the form of: "Which trajectory of the robot do you prefer?". They also proposed an active learning approach by formulating the problem as maximum volume removal optimization. Later, \cite{christiano2017deep} took a similar approach to learn reward functions for reinforcement learning. Most recently, \cite{biyik2018batch} showed it is possible to accelerate the optimization through batch-active learning. They proposed several efficient methods and pointed out the similarity between their ``successive elimination" method and Mat\'ern point processes \cite{matern2013spatial}, which is another repulsive point process. Therefore, we hypothesize that DPPs are a good fit for this problem.

We first briefly summarize the problem. Humans have different preferences over how robots should operate. For example, one can want his/her autonomous car to drive aggressively, whereas another person prefers much safer and defensive driving. These preferences are encoded as a reward function $R$ over trajectory features. These features can be, for example, the speed of the car, the distance to other cars, the heading angle, etc. The main assumption both \cite{sadigh2017active} and \cite{biyik2018batch} had is that the reward is linear in features:
\begin{align*}
R(\phi(\zeta)) = \omega^T\phi(\zeta)
\end{align*}
where $\zeta$ is a trajectory and $\phi(\zeta)$ are its features. The purpose is to efficiently learn $\omega$. By assuming humans follow a softmax noise model \cite{holladay2016active}, they perform Bayesian inference and sample $\omega$. Those samples are then used to actively generate new queries. For that, they formulated the problem as maximum value removal optimization.

Our idea is to use the value of optimization, i.e. expected volume removal amount, as the score associated with the corresponding query. This replaces the uncertainty measure in the classification tasks. As in \cite{biyik2018batch}, we use Euclidean distance between query features (the difference of trajectory features in the query) for diversity, again with a Gaussian kernel.

In \cite{biyik2018batch}, they also employed the heuristic that they can first greedily preselect the most informative $20k$ queries from all $500,\!000$ queries in the dataset, and then select $k$ of them with the methods they propose. For fair comparison and fast computation, we employed the same idea and replaced their method with Active DPP-Mode.

\setcounter{theorem}{1}
\begin{theorem}
	Under the same assumptions as Theorem 3.1 of \cite{biyik2018batch}, the batch query selection with Active DPP-Mode will remove at least $1-\delta$ times as much volume as removed by the best adaptive strategy after $k\ln(1/\delta)$ times as many queries.
\end{theorem}
\begin{proof}
	It was shown in \cite{biyik2018batch} that if a batch-mode active learning method queries the individually most informative pairwise comparison in each batch, then the given performance guarantee will hold due to submodularity \cite{krause2014submodular}. The proof is then complete when we note Active DPP-Mode starts each iteration by adding the query that will remove the most volume in expectation into the batch.
\end{proof}

We quantitatively evaluated our method in comparison with \cite{biyik2018batch}'s methods (greedy, medoids, boundary medoids, successive elimination) for four different environments. First, inspired by their experiments, we simulated a simple linear dynamical system (LDS) with $6$ states, whose means are directly features, and $3$-dimensional control inputs. Second and thirdly, we used MuJoCo physics engine \cite{todorov2012mujoco} to simulate a Fetch mobile manipulator robot \cite{wise2016fetch} where the task is to reach an object on the table without hitting an obstacle as in \cite{palan2019learning}, and a tosser robot where the task is to toss an object into one of two baskets based on preferences. Lastly, we used a driving simulator \cite{sadigh2016planning} to learn different driving preferences. We used the same features as \cite{biyik2018batch} on \emph{Driver} and \emph{Tosser} environments. We used the speed of the manipulator; distance to the table, to the goal object, and to the obstacle as $4$ features in \emph{Fetch} environment. Visuals from the environments are shown in Fig.~\ref{fig:env_visuals}.
\begin{figure}[h]
	\centering
	\vspace{-15px}
	\includegraphics[width=0.6\columnwidth]{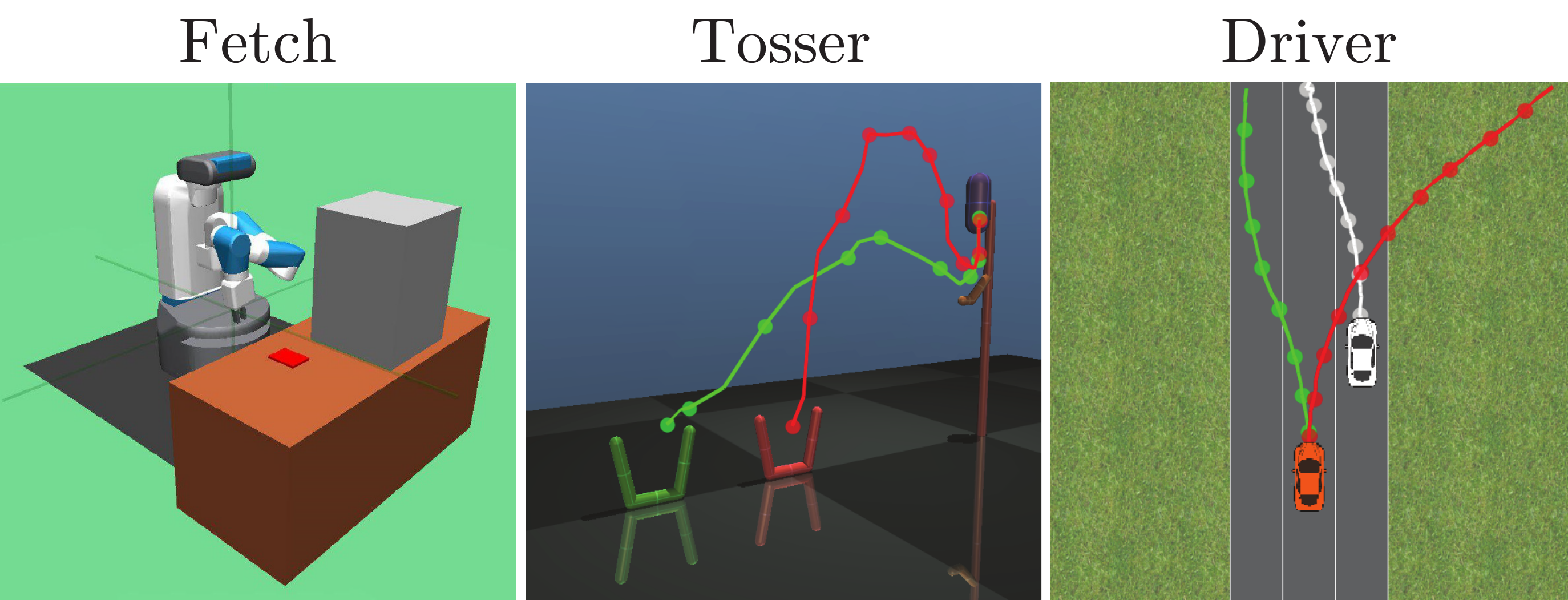}
	\caption{Views from the environments are shown.}
	\vspace{-15px}
	\label{fig:env_visuals}
\end{figure}

For our experiments, we took $k=10$. We also created a new dataset of $100,\!000$ queries for each environment. We randomly generated $200$ different reward functions ($\omega$-vectors), $100$ of which are for tuning $\gamma$ and the remaining $100$ are for tests. This is again for each environment. The same approach can be employed in practice: One can simulate random reward functions for tuning and then deploy the system to learn the reward functions from real users. For both tuning and tests, we simulated noiseless users in order to eliminate the effect of noise in the results. The same approach was taken in \cite{biyik2018batch}. The tuning yielded $\gamma=1$ for LDS and Tosser, $\gamma=4$ for Fetch, and $\gamma=0$ for Driver environments. The details of tuning can be found in Section~\ref{subsec:models_and_tuning} of the Appendix.

\begin{wrapfigure}{r}{0.6\textwidth}
	\centering
	\vspace{-15px}
	\includegraphics[width=0.6\columnwidth]{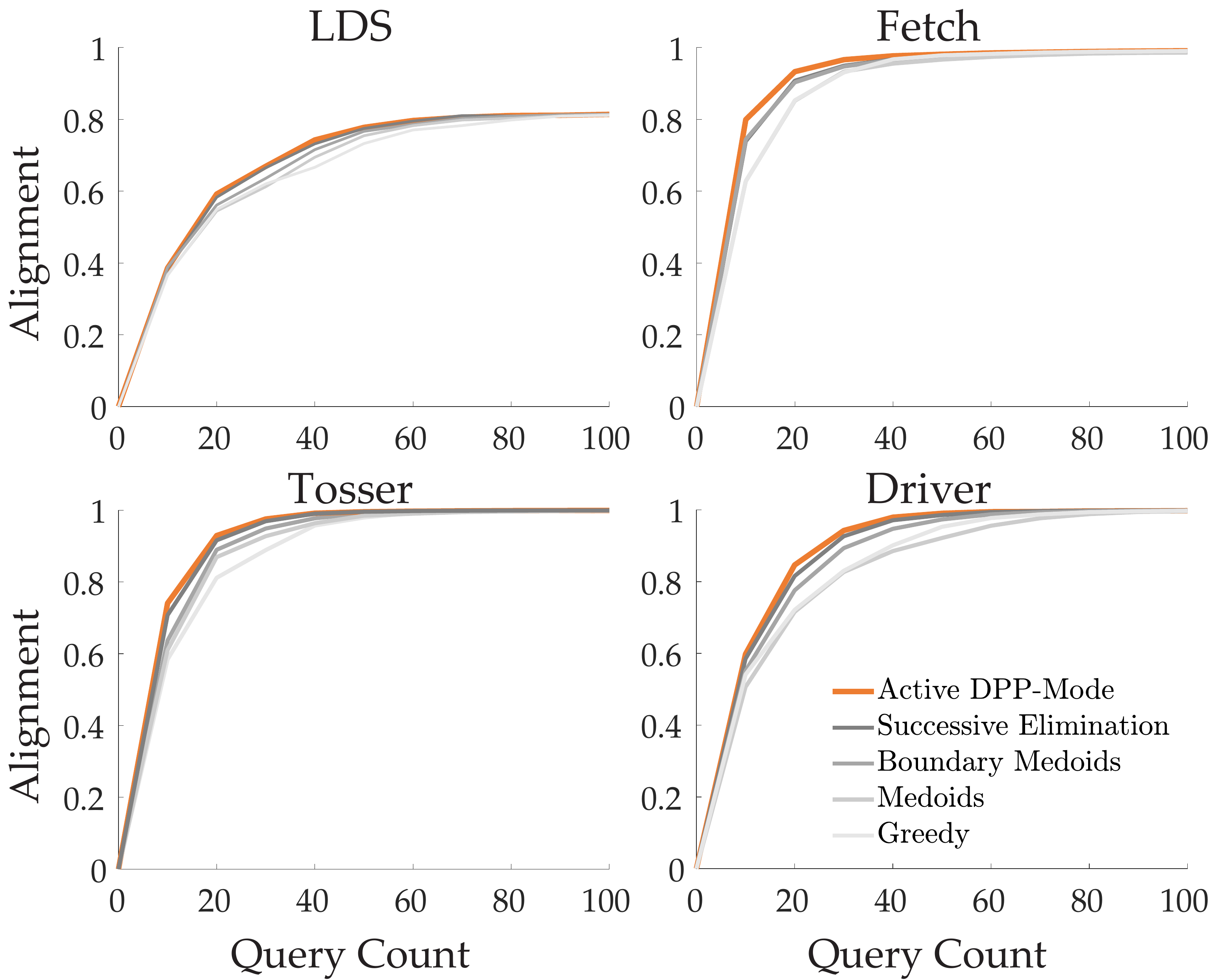}
	\caption{The results of the reward learning task are shown.}
	\vspace{-10px}
	\label{fig:rew_results}
\end{wrapfigure}
We demonstrate the results in Fig.~\ref{fig:rew_results}. The alignment metric is defined as $\omega^T\hat\omega$ where $\hat\omega$ are the estimated weights, and the weights are normalized such that $\norm{\omega}_2=\norm{\hat\omega}_2=1$.

Since the reward functions are paired between the methods, we used Wilcoxon signed-rank tests over the alignment values for significance testing after $30$ queries. While our results confirm the findings of \cite{biyik2018batch} in that successive elimination method outperforms their other alternatives, we observed that Active DPP-Mode significantly outperformed all the methods in all environments ($p<0.05$) except for successive elimination in LDS where both algorithms perform comparably.

\subsection{Models and Tuning}
\label{subsec:models_and_tuning}

\textbf{Hyperparameter Tuning}

We introduced $\sigma$, $\alpha$, $\gamma$, and $\epsilon$ for DPP-based methods. While $\gamma$ and $\epsilon$ are important only for active methods, $\sigma$ and $\alpha$ plays a role in both passive and active techniques. On the other hand, mode-variants eliminate $\alpha$, as it does not affect the results unless trivially $\alpha=0$.

As $\alpha$ and $\gamma$ are enough to adjust the trade-off between diversity and informativeness, we simply set $\sigma$ to be the expected distance between two nearest neighbors when $k$ samples are selected uniformly at random in the space $[0,1]^d$ where $d$ is the number of features of the data samples.

Given an unlabeled dataset $\mathcal{X}$, we cannot try different hyperparameter values, because once we get the labels, we already spend our budget $K$.
To perform tuning for $\alpha$ and $\gamma$, we generate \emph{fake labels} $\mathcal{Y}'$, and tune the hyperparameters on $(\mathcal{X},\mathcal{Y}')$. The procedure for fake label generation can rely on some heuristics or domain expertise.

We visualize the tuning set $(\mathcal{X},\mathcal{Y}')$ we created for the experiments we made with synthetic dataset in Fig.~\ref{fig:synthetic_datasets_tuning}. Note that $\mathcal{Y}$ and $\mathcal{Y}'$ make different number of assignments for each class, and the decision boundaries are completely different, as we do not know such properties of $\mathcal{Y}$ in practice. Using the fake labels, we tune $\alpha$ and $\gamma$ for DPP-based methods, whose results are presented below. 

\begin{wrapfigure}{r}{0.4\textwidth}
	\centering
	\vspace{-15px}
	\includegraphics[width=0.4\columnwidth]{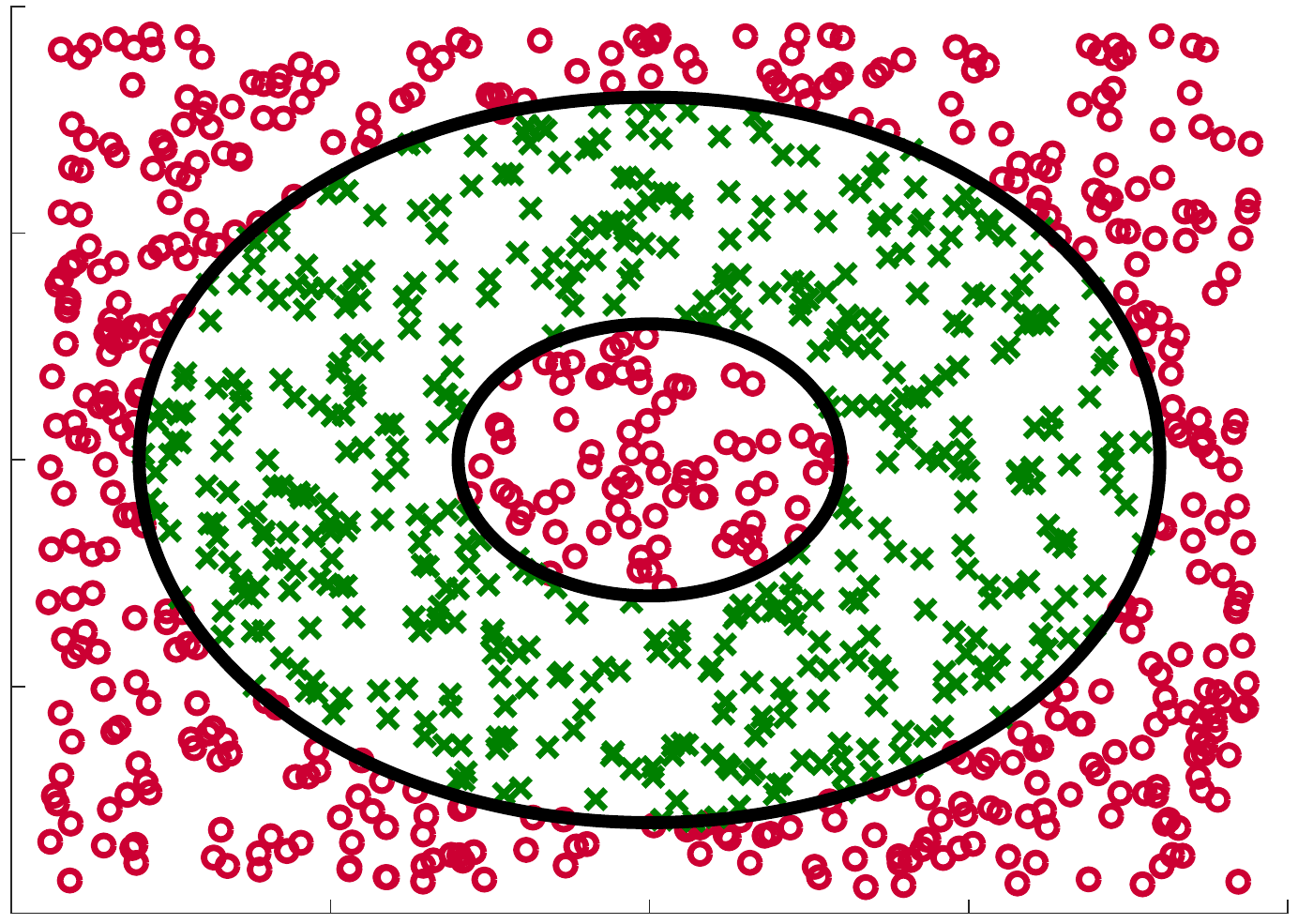}
	\vspace{-12px}
	\caption{Tuning dataset $(\mathcal{X},\mathcal{Y}')$ we used for the classification task on synthetic dataset.}
	\label{fig:synthetic_datasets_tuning}
	\vspace{-22px}
\end{wrapfigure}
For hyperparameter tuning in classification tasks on real datasets, we created $\mathcal{Y}'$ as follows. We selected $C$ samples of the dataset uniformly at random, and called them centroids of distinct classes. We then assigned each of $N$ samples to the class with respect to their closest centroid. We repeated this procedure until all classes have at least $2N/3C$ samples. We provide tuning results below. 

While $\epsilon$ is another parameter that can be optimized using the same technique, we directly use a fixed value of $\epsilon=1/3$. This is because the effect of $\epsilon$ depends on the dispersion of decision boundaries. However, we do not have such information, because the true labels $\mathcal{Y}$ is unknown.

\textbf{Classification on Synthetic Dataset}

For the classifier model, we used a feed-forward one-hidden layer neural network. The hidden layer had $4$ nodes. We used sigmoid function as the nonlinear activation.

Using fake labels, we independently tuned $\alpha$ and $\gamma$ for Active DPP and Passive DPP as described below. We used the same set of hyperparameters for the DPP-mode variants. 

\textbf{\textit{Passive DPP.}} We ran the Passive DPP algorithm on the synthetic dataset $(\mathcal{X},\mathcal{Y}')$ with varying $\alpha$ values in the range $[0,7]$. Due to the fast mixing properties of Monte Carlo Markov Chain (MCMC) for $0\leq\alpha\leq1$, we took $\alpha$'s with $0.1$ step sizes in $[0,1]$, and for the rest of the interval we adopted a step size of $0.5$. We ran the algorithm with each $\alpha$ for $100$ times.

Passive DPP is mostly robust to changes in $\alpha$, when $\alpha\geq2$. This is because the resulting set is diverse enough with large $\alpha$. Hence, we used $\alpha=5$ for the training on $\mathcal{X}$ and assessment on $\mathcal{X}_\textrm{test}$ as described in the paper.

\textbf{\textit{Active DPP.}} Similarly, we ran Active DPP algorithm on the synthetic dataset $(\mathcal{X},\mathcal{Y}')$. This method has two hyperparameters that we tune: $\alpha$ and $\gamma$. We first fixed $\gamma$ to be $5$ and executed runs on varying $\alpha$ again in the range of $[0,7]$. We skipped $\alpha=0$, as it does not enforce diversity and suffers from the redundancy issues. We set the step size to be $1$ for $\alpha\geq1$ and $0.2$ for $\alpha<1$ again due to fast mixing properties. In the second set of tuning experiments, we fixed $\alpha=4$ and varied $\gamma$ in $[0,7]$ with a step size of $1$. In both experiments, we ran the algorithm $100$ times for each $\alpha$ and $\gamma$.

Based on these simulations, we used $\alpha=4$ and $\gamma=5$ for training on $\mathcal{X}$ and for assessment on $\mathcal{X}_\textrm{test}$, and presented the results in the paper.



\textbf{Classification on Real Datasets}

After we saw the comparable results of Passive DPP with Passive DPP-Mode, and Active DPP with Active DPP-Mode on the classification task with synthetic dataset; we decided to continue with the mode variants as they are faster and eliminate the parameter $\alpha$. So on the real dataset classification experiments, we only tune $\gamma$ for Active DPP-Mode.

We created the datasets with fake labels $(\mathcal{X},\mathcal{Y}')$ as we describe in the paper. For the classifier model, we used feed-forward neural networks. The list below specifies the network structures for each dataset, which were chosen based on the dataset complexity.
\begin{itemize}[nosep]
	\item \textbf{Compressed MNIST:} Input(5), Hidden(10), Output(10)
	\item \textbf{Compressed Fashion-MNIST:} Input(5), Hidden(10), Output(10)
	\item \textbf{Morphological:} Input(6), Hidden(12), Output(10)
	\item \textbf{Segment:} Input(18), Output(7)
	\item \textbf{WRN:} Input(4), Hidden(8), Output(4)
	\item \textbf{Blood:} Input(4), Hidden(8), Hidden(8), Output(2)
\end{itemize}

As we have observed that DPP methods are mostly very robust to the changes in $\alpha$ and $\gamma$, and results of the synthetic dataset experiment have showed DPP-Mode achieves comparable performance to corresponding DPP methods; we used DPP-Mode variants to tune the hyperparameter $\gamma$. 

\textbf{\textit{Active DPP-Mode}}. We tuned $\gamma$ in the range of $[0,7]$ with step size $1$ using the fake labels for each real dataset. We experimented each $\gamma$ value $100$ times. The results suggest the algorithm is somewhat robust to different values of $\gamma$, too, in the given interval. We used the following $\gamma$ values for the experiments with original datasets, as they gave the highest average accuracies on tuning sets:
\begin{itemize}[nosep]
	\item \textbf{Compressed MNIST:} $\gamma=5$
	\item \textbf{Compressed Fashion-MNIST:} $\gamma=3$
	\item \textbf{Morphological:} $\gamma=3$
	\item \textbf{Segment:} $\gamma=1$
	\item \textbf{WRN:} $\gamma=1$
	\item \textbf{Blood:} $\gamma=2$
\end{itemize}

\textbf{Preference-based Reward Learning}

Again due to aforementioned reasons, we used Active DPP-Mode instead of Active DPP. We have not tried passive methods in this setting, because all the state-of-the-art methods that we compare use active learning techniques.

\begin{wrapfigure}{r}{0.6\textwidth}
	\centering
	\vspace{-25px}
	\includegraphics[width=0.6\columnwidth]{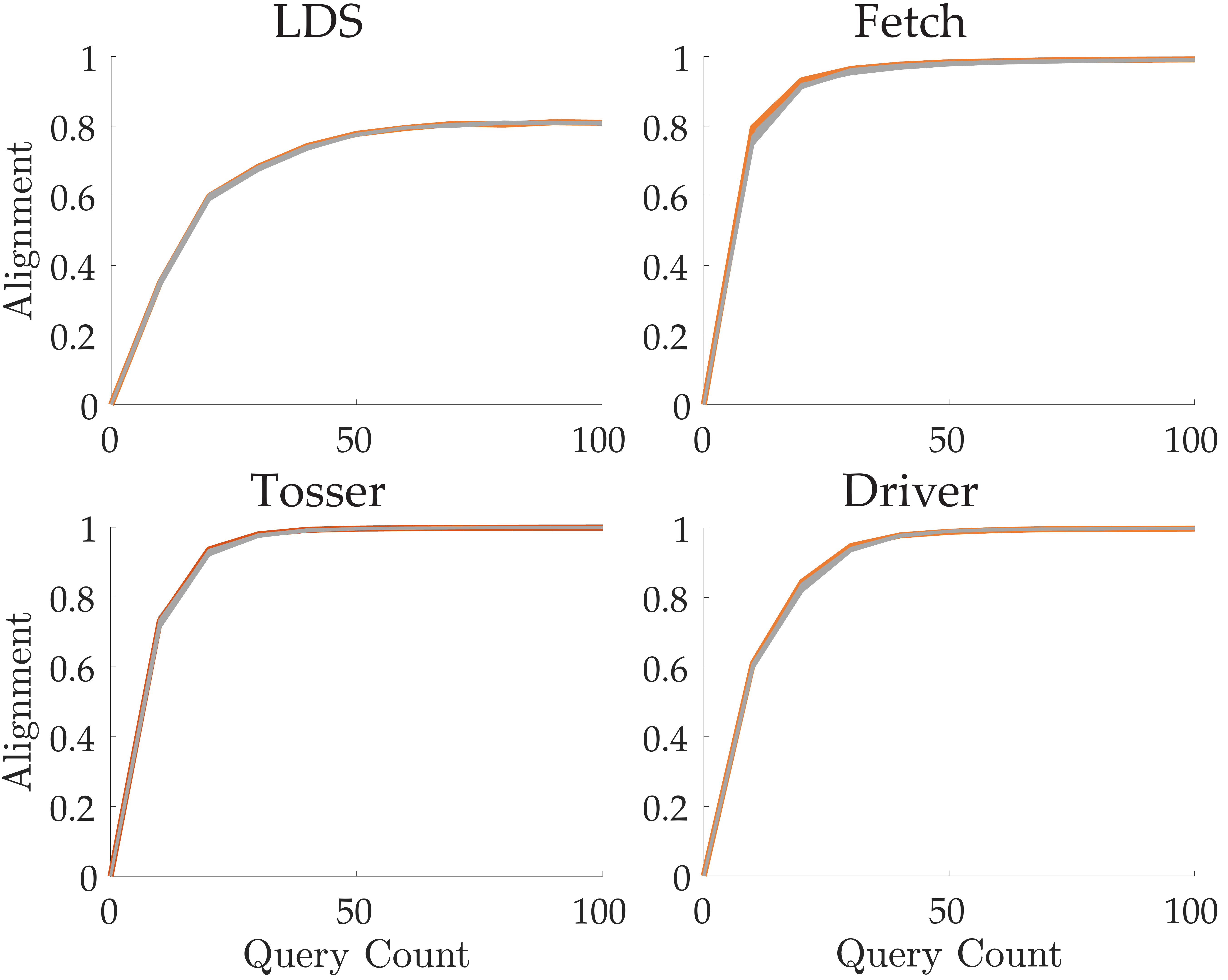}
	\vspace{-15px}
	\caption{Tuning results for Active DPP-Mode for the reward learning tasks.}
	\vspace{-20px}
	\label{fig:reward_learning}
\end{wrapfigure}
\textbf{\textit{Active DPP-Mode.}} We tuned $\gamma$ separately for the linear dynamical system (LDS), the Fetch mobile manipulator \cite{wise2016fetch}, Tosser from MuJoCo \cite{todorov2012mujoco}, and driving simulation \cite{sadigh2016planning} environments where each $\gamma$ has been experimented $100$ times with different true reward functions.

Fig.~\ref{fig:reward_learning} shows how the alignment value, defined in the paper, changes with different number of queries. We desire having large alignment values within a small number of queries. We highlighted the selected $\gamma$ parameters in the plots.

As can be seen from the results, the effect of $\gamma$ on performance was slight, and it was hard to select the ``best" $\gamma$. We qualitatively selected $\gamma=1$ for LDS and Tosser, $\gamma=4$ for Fetch, and $\gamma=0$ for Driver based on their slight advantage in learning rate with respect to query counts.

\subsection{Computation Infrastructure}
We used 2 different Ubuntu machines whose details are given below.
\begin{itemize}[nosep]
	\item Ubuntu 16.04, Intel$^{\textregistered}$ Xeon$^{\textregistered}$ Silver 4114 CPU @ 2.20GHz, 40 CPUs, 125GB RAM
	\item Ubuntu 18.04.2 LTS, Intel$^{\textregistered}$ Xeon$^{\textregistered}$ CPU @ 2.20GHz, 32 CPUs, 32GB RAM
\end{itemize}

\end{document}